\documentclass{article}

\usepackage{arxiv}

\usepackage[utf8]{inputenc} 
\usepackage[T1]{fontenc}    
\usepackage{hyperref}       
\usepackage{url}            
\usepackage{booktabs}       
\usepackage{amsfonts}       
\usepackage{nicefrac}       
\usepackage{microtype}      
\usepackage{lipsum}		
\usepackage{graphicx}
\usepackage{natbib}
\usepackage{doi}
\usepackage{amssymb}

\author{Jules Sintes \\
	Inria\\
        DIENS, École Normale Supérieure, PSL University\\
        Paris, France\\
	\texttt{jules.sintes@inria.fr} \\
	\And
	Ana Bušić \\
        Inria\\
        DIENS, École Normale Supérieure, PSL University\\
        Paris, France\\
	\And
        Jiamin Zhu \\
        IFP Energies nouvelles\\
       Rueil-Malmaison, France\\
}

\renewcommand{\headeright}{}
\renewcommand{\undertitle}{}
\renewcommand{\shorttitle}{Structural Equivalence and Learning Dynamics in Delayed MARL}

\hypersetup{
pdftitle={Structural Equivalence and Learning Dynamics in Delayed MARL},
pdfsubject={cs.LG, cs.AI},
pdfauthor={Jules Sintes, Ana Bušić, Jiamin Zhu},
pdfkeywords={Reinforcement Learning, Multi-Agent, Dec-POMDP, Delays, Theory},
}

\usepackage{graphicx}           
\usepackage{subcaption}         
\usepackage[space]{grffile}     

\usepackage{multirow}
\usepackage{amsmath}
\usepackage{amsthm}
\usepackage{algorithm}
\usepackage{algpseudocode}
\usepackage{tikz}
\usepackage{standalone}
\usepackage{pgfplots}
\pgfplotsset{compat=newest}
\usepgfplotslibrary{fillbetween} 

\newtheorem{assumption}{Assumption}
\newtheorem{theorem}{Theorem}
\newtheorem{lemma}{Lemma}
\newtheorem{definition}{Definition}
\newtheorem{proposition}{Proposition}
\newtheorem{corollary}{Corollary}

\definecolor{ad_main}{RGB}{0, 80, 157} 
\definecolor{ad_cold}{RGB}{70, 160, 230} 
\definecolor{od_color}{RGB}{214, 80, 20}

\begin{document}
\title{Structural Equivalence and Learning Dynamics in Delayed Cooperative Multi-Agent Reinforcement Learning}
\maketitle
\begin{abstract}
    We formally establish the equivalence between Observation Delay (OD) and Action Delay (AD) in cooperative partially observable multi-agent systems using observation-action histories. We show that both systems generate identical admissible joint-policy sets, and their induced state-action-observation trajectories are identical in distribution, leading to identical optimal solutions in Decentralized Partially Observable Markov Decision Processes (Dec-POMDPs). This formally generalizes existing infinite-horizon single-agent results to any-horizon partially observable cooperative multi-agent problems with decentralized policy execution, and allows any mixed-delay configuration to be reduced to a pure OD system. We further prove that in Transition-Independent MDPs (TI-MDPs), the observation-action history reduces to a tractable minimal local augmented state.
    However, we show through numerical experiments that although the optimal solution spaces are structurally isomorphic, the practical learning dynamics are fundamentally different. First, using the minimal local augmented state, the equivalence no longer holds when transitions are not independent. Second, operational constraints and causal credit-assignment errors in Temporal Difference (TD) algorithms induce different learning behaviors across regimes. Finally, leveraging this structural equivalence to bypass these learning challenges, we demonstrate successful multi-agent zero-shot policy transfer from OD to AD, paving the way for unified, efficient solution methods in complex delayed systems.
\end{abstract}
\keywords{Reinforcement Learning, Multi-Agent, Dec-POMDP, Delays, Theory}

\section{Introduction}

Delay is a central challenge in controllable systems. In robotics and its applications, such as autonomous driving, control delay arises from computational or sensor latency \citep{mahmood2018setting,fuchs2021super}. This problem is even more prevalent in decentralized multi-agent systems, where each controlled agent may have a different delay. This increases the difficulty of coordinated control due to communication delays in network systems \citep{guo2023safe,xia2022optimal,wei2023drl} or physical inertia of the environment, such as in Wind Farms \citep{monroc2022delay}. Reinforcement Learning (RL) has emerged as a powerful tool for efficient control of increasingly complex systems. While it has demonstrated impressive empirical results in many domains, learning under delay remains highly difficult. Formally, delays can be introduced in RL problems described through the Markov Decision Process (MDP) framework. In this context, \cite{katsikopoulos2003markov} showed that OD and AD are formally equivalent in single-agent infinite-horizon problems with asynchronous cost collection. They approached delay equivalence through the lens of state augmentation and cost reshaping. Their method constructs a surrogate MDP where the optimal policy order is preserved.\\

In this paper, we propose a stronger form of equivalence based on \textit{observation-action histories}, effectively separating the structural equivalence from the learning problem. We show that under a proper initialization, the effective information set available to an agent at the moment of decision for an action is identical in both delay regimes. Consequently, the two regimes generate statistically identical trajectories, establishing a structural isomorphism. This implies that the equivalence holds for any horizon, any cooperative reward objective and structure (e.g., average reward, sparse reward), and under partial observability. Our results strengthen and extend the existing equivalence claim in single-agent scenarios and hold in cooperative multi-agent settings with decentralized policy execution. Consequently, we provide corollaries regarding the additivity of OD and AD when agents are subject to both delays, and show that any system with mixed-regime heterogeneous delays reduces to a pure OD system. Additionally, we propose sufficient statistics for Transition Independent dec-MDP (TI dec-MDP), enabling a more tractable reduction to an undelayed MDP, extending the augmented-state method from the single-agent to the multi-agent case. However, we demonstrate that structural equivalence of the solution space does not imply symmetric learning dynamics. While convergence towards equivalent policy holds, we empirically identify specific operational constraints, such as initialization constraints or causal credit-assignment errors, that render the AD regime harder to learn under standard Temporal Difference algorithms. Furthermore, we provide an empirical counter-example where a tractable augmented state leads to a different optimal policy when breaking Transition Independence. Finally, we demonstrate that the equivalence of policy set still has some practicality by effectively performing zero-shot transfer from OD regime learned policy to AD, on a complex multi-agent coordinated control environment (Multiwalker). \\

The paper is organized as follows. First, we give a comprehensive overview of existing results in the field of delayed reinforcement learning, introducing the core concepts used in this paper. Second, we derive the structural equivalence results and the associated corollaries. Then, we study the practical gap in the learning dynamics under OD and AD through experiments and additional theoretical results.

\section{Related Work}

To the best of our knowledge, there are very few theoretical results regarding delays in Multi-Agent Reinforcement Learning (MARL). Thus, unless explicitly mentioned, results cited in this section are derived for single-agent RL. This work focuses on Observation Delay (OD) and Action Delay (AD). Hence, our theory sets aside the reward collection delay problem. As highlighted by \cite{liotet2022delays}, the problem of reward collection delay mainly affects the performance during the learning phase. It is widely discussed theoretically in terms of regret in the bandit literature \citep{joulani2013online}. Here, we only focus on delay regimes as part of the environment that affects the performance of the policy execution. However, we shall discuss some practicality regarding reward collection when dealing with OD or AD. \\

Delays in RL often remain an applied problem in the literature, and most of the existing work focuses on methods to effectively handle delays. RL algorithms for handling delayed systems fall into three categories: Model-based \citep{walsh2007planning,derman2021acting,firoiu2018human}, memoryless \citep{schuitema2010control,hester2013texplore}, and augmented state. The latter is the most straightforward and at the core of OD and AD's equivalence claim. In single-agent settings, \cite{bertsekas1987dynamic} shows that an MDP with constant OD reduces to an undelayed MDP with an augmented space taking into account past decided actions. These results have been extended to AD, and stochastic delay regimes with asynchronous cost collection by \cite{katsikopoulos2003markov}. A direct consequence of their augmented state method is the functional equivalence of OD and AD. However, they assume \textbf{infinite-horizon} and a \textbf{cost collection delay} larger than the OD (or AD). In single-agent, \cite{chen2021delay} propose a delay-aware model-based method for continuous control and use the equivalence argument to justify focusing on AD regime, arguing that their results hold for OD but omitting any discussion regarding the underlying assumption for \cite{katsikopoulos2003markov} equivalence to hold in their settings. The method is extended to multi-agent systems, both competitive and cooperative, with the same argument \citep{chen2020delay}. In \cite{liu2024delay}, they use the same argument and choose to develop their algorithm for problems formulated with AD. However, applying this equivalence to multi-agent settings is risky. The original proof relies on specific cost-discounting mechanics that do not trivially extend to decentralized policies or finite-horizon cooperative tasks. It remains an open question whether a rigorous structural equivalence holds in decentralized multi-agent systems.\\

Overall, this work provides a theoretical backbone for the equivalence of OD and AD in cooperative decentralized multi-agent systems, also giving a stronger isomorphism for the single-agent case. Furthermore, we clarify existing claims from the literature, using the equivalence as well as superset arguments to justify applied work on delayed systems.

\section{Background}

\subsection{Cooperative Multi-agent Reinforcement Learning}

The general framework for MARL is the Markov Game, which generalizes the single-agent setting of the MDP to $N$ agents. While in a single-agent MDP (standard RL), the environment is stationary, Markov Games introduce non-stationarity from the perspective of individual agents.
Here, we are focusing on cooperative MARL with a shared reward. Then, the standard mathematical framework is the Decentralized Partially Observable Markov Decision Process (Dec-POMDP). These are known to be N-EXP hard to solve \citep{oliehoek2016concise}. \\

\noindent We consider the standard formulation of Dec-POMDP as the tuple :
$$\mathcal{M} = \langle \mathcal{N}, \mathcal{S}, \mathbf{A}, P, \mathbf{O}, \Omega, R, \gamma \rangle$$
where 
\begin{itemize}
    \item $\mathcal{N}$ is the set of $n$ agents. We use indices $i$ to refer to the $i$-th agent.
    \item $\mathcal{S}$ is the joint state space.
    \item $\mathbf{A} = \times_{i \in \mathcal{N}} \mathbf{A}_i$ is the joint action space, expressed as the product of local action spaces (which can be different for each agent). At every step $t$, each agent $i$ takes an action $a_i \in A_i$ and the joint action is the vector $\mathbf{a}=\langle a_1, \dots, a_n \rangle$.
    \item $P$ is the transition probability function specifying $\text{Pr}(s' \mid s,\mathbf{a})$, the probability of the state $s\prime$ when the agents execute the joint action $a$ from $s$.
    \item $\mathbf{\Omega}$ is the joint observation space. Similar to the joint action space, $\mathbf{\Omega} = \times_{i \in \mathcal{N}} \Omega_i $. 
    \item $O$ is the observation probability function mapping the current state to a joint observation $\mathbf{o}=\langle o_1, \dots, o_n \rangle$ defined as the vector of local observations for each agent. The function specifies $\text{Pr}(\mathbf{o}|s,\mathbf{a})$, mapping a state-action pair to a joint observation.
    \item $R$ is the immediate reward function. The reward function $R$ characterizes the immediate feedback generated by the environment and shared across agents with $R: \mathcal{S} \times \mathbf{A} \to \mathbb{R}$. The scalar reward $r_t$ received at time $t$ is defined as $r_t = R(s_t, \mathbf{a}_t)$.
    \item $\gamma$ is the discount factor, with $\gamma \in (0,1]$ and $\gamma < 1$ for infinite horizon problems.
\end{itemize} 

The tuple $\mathcal{M}$ fully defines the Dec-POMDP framework. We assume that agents are synchronous and act simultaneously based on local histories. At each step, an agent $i$ receives the observation $o_i \in \Omega_i$ computed with $O(s_t,\mathbf{a}_{t-1})$ and takes an action $a_{i,t} \in A_i$. The state $s_{t+1}$ is computed with $P(s_t,\mathbf{a}_t)$. To ensure clarity throughout our analysis, we explicitly distinguish between two temporal references. Our primary reference is the \textbf{effective time} $t$, which acts as an absolute, centralized environment clock tracking the exact moment an action takes effect. In contrast, the \textbf{decision time} is subjective and agent-specific, referring strictly to the moment an individual agent computes and chooses its action.

\subsection{Modeling Delays}

We consider the standard delay formulations for MDPs as defined by \cite{katsikopoulos2003markov}. Let $k \in \mathbb{N}$ denote the integer length of a delay. Delays are categorized into three distinct types:

\begin{itemize}
    \item \textbf{Observation Delay:} State information is not available instantaneously; at time $t$, the observation of a delayed agent is generated from the state $s_{t-k}$.
    \item \textbf{Action Delay:} An action selected at decision time step $t$ takes effect at the later step $t+k$.
    \item \textbf{Reward Delay:} The reward induced by an action is collected after a number of steps.
\end{itemize}

Here, we consider only Observation and Action delay and do not make any assumption on reward delay. In the rest of the paper, we denote $k$ the length of the delay in both delay regimes. In a multi-agent setting, each agent can be subject to a different type and length of delay. We will first consider the case where all agents are subject to the same type of delay (OD or AD).

\section{Formal Equivalence of Heterogeneous Delays in Multi-Agent settings}

\subsection{Definitions and Notation}

\begin{itemize}
    \item Let $\mathcal{M}$ be a Dec-POMDP with $n$ agents.
    \item Let $\mathbf{k} = (k_1, \dots, k_n) \in \mathbb{N}^n$ be a \textbf{Delay Configuration Vector} giving the delay value for each agent.
\end{itemize}

First, we define the history that encapsulates the exact state of knowledge available to an agent $i$ to determine an action taking effect at time $t$, assuming an information gap of size $k_i$.

\begin{definition}[Observation-Action History]\label{def:history}
For any agent $i \in \mathcal{N}$ subject to delay $k_i$, we define the \textit{observation-action history} $h_i^{(t, k_i)}$ associated with the effective time $t \ge k_i$ as:
\begin{equation}
     h_i^{(t, k_i)} \triangleq \left( o_{i, 0:t-k_i}, \quad a_{i, 0:t-1} \right)
\end{equation}
where the indices specify a sequence of time-ordered local observations and executed local actions.
\end{definition}

Here, $t$ refers to the exact effective time of the action currently being decided. We define this indexing for the clarity of the derivation of the structural equivalence. Hence, this history represents the knowledge in both regimes for an action taking effect at time $t$:
\begin{itemize}
    \item \textbf{OD:} The decision is made at time $t$. The agent  observes the delayed state up to $o_{t-k_i}$ and remembers its own immediate past actions up to $t-1$.
    \item \textbf{AD:} The decision for action $a_{i,t}$ must be made at time $\tau = t - k_i$. At this moment, the agent observes up to the current state $o_{\tau}$ (which is exactly $o_{t-k_i}$). Furthermore, it has a memory of its previously decided actions that are sitting in its buffer (see Definition \ref{def:action-buffer}) and will take effect up to $\tau + k_i - 1$ (which is exactly $t-1$). 
\end{itemize}
This definition describes the set of information available for a decision for both OD and AD and is at the core of the structural equivalence theorem we derive in this section.



\begin{definition}[Action Buffer]\label{def:action-buffer}
For any agent $i \in \mathcal{N}$ subject to an action delay $k_i \ge 1$, the \textit{action buffer} at decision time $\tau$ contains the pending actions that will take effect in the next $k_i$ steps. Formally, using effective-time indices:
\begin{equation}
    \mathbf{a}_{i, \tau : \tau+k_i-1} \triangleq (a_{i, \tau}, a_{i, \tau+1}, \dots, a_{i, \tau+k_i-1}) \in \mathcal{A}_i^{k_i}
\end{equation}
\end{definition}

\begin{definition}[Local Policy]\label{def:local-policy}
A (local) policy for agent $i$ is defined as a function
\[
\pi_i : \mathcal{H}_i \to \Delta(\mathcal{A}_i),
\]
where $\mathcal{H}_i$ denotes the set of observation-action histories $h_i^{(t,k_i)}$ available to agent $i$ at effective time $t$. For each observation-action history $h_i^{(t,k_i)} \in \mathcal{H}_i$, the policy $\pi_i(\cdot \mid h_i^{(t,k_i)})$ specifies the conditional distribution over actions chosen at time $t$.
\end{definition}

A local policy's distribution over actions depends solely on the observation-action history associated with the effective time $t$.\\

Throughout this section, we analyze the stochastic process induced by a fixed joint policy $\pi$. Agents do not adapt, re-plan, or condition on counterfactual decision times during execution. Consequently, all stochasticity in the system arises solely from the environment dynamics, observation probability function, and the policies’ internal randomness.

\subsection{Information Equivalence}

In this section, we construct the strict equivalence of delay regimes with trajectory isomorphism. We begin with the set of necessary assumptions


\begin{assumption}[FIFO Buffer with State-Independent Action Space]\label{assumption:deterministic-buffer}
For all agents $i \in \mathcal{N}$, the set of feasible actions $\mathcal{A}_i$ is independent of the current environment state $s_t$. Additionally, for an agent under AD, the local action buffer as in Definition \ref{def:action-buffer} strictly follows the First-In First-Out (FIFO) discipline.
\end{assumption}

This assumption is made without loss of generality. If an agent selects an action that is technically invalid upon execution at time $t$, the environment's transition dynamics simply handle it naturally (e.g., by producing no change in state). This ensures that any action in the buffer can always be mathematically executed, preserving the framework without limiting the types of problems it can model.

\begin{assumption}[Decentralized Execution of the Joint Policy]\label{assumption:dec-joint}
We consider the set of joint policies defined as the Cartesian product of independent local policies:
\begin{equation}
    \boldsymbol{\pi} = \langle \pi_1, \dots, \pi_n \rangle, \quad \text{where } \pi_i \text{ are as in Definition \ref{def:local-policy}}
\end{equation}
\end{assumption}

The decentralized execution assumption is relevant as the standard setting for multi-agent reinforcement learning. A centralized execution of the joint policy would trivially reduce to the single-agent case and is not of interest in this paper. Policies are defined as mapping from observation-action histories to probability distributions of actions. The timing of when an action is \textit{decided} is not part of the model and serves only as an interpretation of the information structure.\\

By defining $\pi_i$ strictly as a mapping from the observation-action history $h_i$ to the action space, we ensure that the decision-making logic is invariant to the effective time of computation. The policy $\pi_i$ functions as a "lookup table" or function approximator that outputs an action given an information state, regardless of whether that lookup occurs at $t$ (OD) or $t-k$ (AD).\\

If policies were allowed to condition explicitly on the "wall-clock time of decision" rather than the "effective time of the history," the two regimes would require distinct policy spaces, preventing a direct bijection.

\begin{assumption}[No Information Leakage Across Agents]\label{assumption:no-leakage}
For each agent $i \in \mathcal{N}$ and effective time $t$, the joint observation $\mathbf{o}_{t}$ is conditionally independent of other agents’ private histories, action buffers, decision times, or delay mechanisms, given the environment state $s_t$:
\begin{equation}
    Pr(\mathbf{o}_{t} \mid s_t, \mathbf a_{t-1}, \{\ h_{i}^{(t,k_i)}\}_{i \in \mathcal{N}})
= Pr(\mathbf{o}_{t} \mid s_t, \mathbf a_{t-1}) = O(s_t, \mathbf a_{t-1})
\end{equation}
and observations at time $t$ depend only on actions effective at times strictly less than $t$.
\end{assumption}

This assumption is required to prevent informational paradoxes between the two regimes.
\begin{itemize}
    \item \textbf{Buffer Privacy:} In the Action Delay regime, future actions $a_{t:t+k}$ are pre-computed and stored in the agent's private buffer. In the Observation Delay regime, these actions have not yet been computed. If another agent could observe this private buffer (violating the assumption) or even infer any other agent's private state, it would effectively possess information about the future that is unavailable in the OD regime.
    \item \textbf{Causal Consistency:} By enforcing that observations depend strictly on the environment state $s_t$ (which reflects only past effective actions $\mathbf{a}_{t-1}$), we ensure that the "visible" history of the process is identical regardless of whether actions are buffered at the agent (AD) or "buffered" by the flow of time (OD).
\end{itemize}

\begin{assumption}[Consistent Initialization of Information Structures]\label{assumption:init}
    Let $\mathcal{I}_0$ denote the initial common information (e.g., the prior belief $b_0$ or empty history $\emptyset$) available at $t=0$. We assume the initial conditions for both regimes are informationally equivalent:
    \begin{itemize}
        \item In the \textbf{OD} regime, for all $t < k$, the agent's available history is restricted to $h_t = \mathcal{I}_0$. The actions $a_0, \dots, a_{k-1}$ are generated by the policy $\pi_{OD}(\cdot \mid \mathcal{I}_0)$ at $t=0$ and $\pi_{OD}(\cdot \mid \mathcal{I}_0, a_{i,0:t-1})$ up to $k$.
        \item In the \textbf{AD} regime, the initial action buffer $\mathbf{a}_{buffer} = (a_0, \dots, a_{k-1})$ is initialized at $t=0$ by sampling the policy $\pi(\cdot \mid \mathcal{I}_0)$ conditioned on the initial information. This is needed to ensure consistency with the OD regime's first actions.
    \end{itemize}
    Consequently, the initial sequence of actions in both regimes is optimized with respect to the same prior $\mathcal{I}_0$.
\end{assumption}

This assumption is arguably the strongest in our settings. It ensures that we have a symmetry of both delay regimes by specifying an initialization phase during the length of the delay for $t<k_i$ during which agents get the same information set when deciding an action. In practice:
\begin{itemize}
    \item \textbf{OD}: During the first steps, an agent observes an initialization information set $\mathcal{I}_0$, and the local policy is conditioned on this set and the history of decided actions until the agent receives its first delayed observation at $t=k_i$.
    \item \textbf{AD}: To make it consistent and build an equivalent regime during the "initialization phase", we allow agents with action delay to pre-compute their action buffer according to their local policy and the same initialization information set $\mathcal{I}_0$.
\end{itemize}

The initialization information set $\mathcal{I}_0$ can be arbitrarily chosen. For example, it can be either an empty observation or the observation of the system at $t=0$, it does not matter, as long as it is the same in both regimes. Overall, we enforce the first actions to be computed based on the same set of information. In AD, this assumption might be more constraining to enforce. However, it is not practically irrelevant. We further discuss the necessity of this assumption in Section \ref{section:practical-gap}.\\

We define the \textit{effective information set} as the maximum subset of the system's history that an agent can actively access to compute its policy at decision time $t$. This leads to the following lemmas regarding information sets in both delay regimes:

\begin{lemma}[Information Sufficiency for Action Delay]
For an agent $i$ subject to AD of length $k_i$, the \textit{Effective Information Set} $\mathcal{I}_{i, AD}^{(t)}$ utilized to determine $a_{i, t}$ is:
\begin{equation}
\begin{split}
    \mathcal{I}_{i, AD}^{(t)} &= \left( o_{i, 0:\tau}, \quad a_{i, 0:\tau-1}^{\text{decided}} \right) \\
    &= \left( o_{i, 0:t-k_i}, \quad a_{i, 0:t-1} \right) \\
    &\equiv h_i^{(t, k_i)}
\end{split}
\end{equation}
\end{lemma}

\begin{proof}
    Consider an agent $i$ subject to AD $k_i$. To determine the action $a_{i,t}$ that takes effect at execution time $t$, the agent must make its decision at time $\tau = t-k_i$. 
    
    At this decision time $\tau$, the agent accesses all observations up to $\tau$ ($o_{i, 0:\tau}$) and a complete memory of its own past decisions ($a_{i, 0:\tau-1}^{\text{decided}}$). By Assumption \ref{assumption:init}, the first $k_i$ executed actions are drawn from the shared prior to prevent initialization asymmetry. This ensures the initial buffer aligns seamlessly with the rest of the agent's decision process. From $t=k_i$ onward, the executed actions are strictly the decisions the agent made starting at $\tau=0$. 
    
    Therefore, the sequence of decided actions up to $\tau-1$ maps perfectly to the sequence of effective actions up to $t-1$. Substituting $\tau = t-k_i$ yields $(o_{i, 0:t-k_i}, a_{i, 0:t-1})$, which is exactly $h_i^{(t, k_i)}$.
\end{proof}

\begin{lemma}[Information Sufficiency for Observation Delay]
For an agent $i$ subject to OD of length $k_i$, the \textit{Effective Information Set} $\mathcal{I}_{i, OD}^{(t)}$ is:
\begin{equation}
    \mathcal{I}_{i, OD}^{(t)} \equiv (o_{i, 0:t-k_i}, \quad a_{i, 0:t-1}) \equiv h_i^{(t, k_i)}
\end{equation}
\end{lemma}

\begin{proof}
    Consider an agent $i$ subject to OD $k_i$. Unlike AD, the decision time and execution time are identical ($\tau = t$). 
    
    However, the environment artificially delays the observations. At time $t$, the agent only receives observations up to $t-k_i$ ($o_{i, 0:t-k_i}$). Crucially, the agent's internal memory is not delayed; it remembers all actions it has executed up to $t-1$ ($a_{i, 0:t-1}$). This combination exactly forms the history $h_i^{(t, k_i)}$.
\end{proof}

Lemmas 1 and 2 establish that, for each agent $i$ and effective time $t$, the marginal information available to determine the action taking effect at time $t$ is identical under AD and OD, and is fully characterized by the observation-action history $h_i^{(t, k_i)}$. Especially, this equivalence concerns each agent’s local information and does not assume or require equality of joint information across agents at intermediate decision times. 


    

\subsection{Main Result}

The key observation is that, under both delay mechanisms, the action taking effect at effective time $t$ is a measurable function of the same observation-action history for each agent. Since observations depend only on past actions and the environment state, the two regimes generate identical distributions of actions at every time step, and hence identical trajectory laws by induction. Therefore, we derive the following theorem:

\begin{theorem}[Joint Policy Equivalence under Heterogeneous Delays]\label{theorem:main-equ}

Let the delay configuration vector be denoted $\mathbf{k} = (k_1,\dots,k_n)$, and let
$\Pi_{AD}(\mathbf{k})$ and $\Pi_{OD}(\mathbf{k})$ denote the sets of admissible
joint policies for the AD and OD regimes,
respectively. Under 
Assumptions \ref{assumption:deterministic-buffer},
\ref{assumption:dec-joint},
\ref{assumption:no-leakage}, and
\ref{assumption:init},
there exists a bijection
\begin{equation}
    \Phi : \Pi_{AD}(\mathbf{k}) \to \Pi_{OD}(\mathbf{k})
\end{equation}
such that, for any joint policy $\boldsymbol{\pi} \in \Pi_{AD}(\mathbf{k})$, the stochastic
process over state--action--observation trajectories $(s_t, \mathbf{a}_t, \mathbf{o}_t)_{t \ge 0}$ induced by $\boldsymbol{\pi}$ in the AD regime is identical to the stochastic process induced by $\Phi(\boldsymbol{\pi})$ in the OD regime. In particular, for every finite horizon $T$, the joint distributions of
\[
(s_0,\mathbf{a}_0,\mathbf{o}_0,\dots,s_T,\mathbf{a}_T,\mathbf{o}_T)
\]
coincide under the two regimes.
\end{theorem}

To formally prove Theorem \ref{theorem:main-equ}, we must establish two properties: first, there exists a bijection $\Phi$ between the admissible joint-policy spaces of the AD and OD regimes; and second, this bijection preserves the induced stochastic process over the state-action-observation trajectories. \\

\begin{proof}

\textbf{Part 1: Isomorphism of the Policy Spaces --} 
By Assumption \ref{assumption:dec-joint}, the execution of a joint policy is decentralized and factorizes as the product of independent local policies, $\boldsymbol{\pi} = \langle \pi_1, \dots, \pi_n \rangle$. Each local policy $\pi_i$ is defined as in Definition \ref{def:local-policy}.

By Lemmas~1 and~2, for each agent $i$ and effective time $t \ge k_i$, the effective information available to determine the action that takes effect at time $t$ is identical in both the AD and OD regimes: the observation-action history $h_i^{(t,k_i)}$. 

Since the effective information sets are identical at the moment of decision for any effective time $t$, the admissible local policies for agent $i$ are identical in the AD and OD regimes. Both are functions of the form:
\[
\pi_i : \mathcal{H}_i \to \Delta(\mathcal A_i)
\]
Let $\phi_i$ be the bijection map on this space of local policies. We can construct a global mapping $\Phi$ such that:
\[
\Phi(\boldsymbol{\pi}) \triangleq \langle \phi_1(\pi_1), \dots, \phi_n(\pi_n) \rangle
\]
Because each $\phi_i$ is a bijection map over identical local policy spaces, $\Phi$ forms a strict bijection between the joint policy spaces $\Pi_{AD}(\mathbf{k})$ and $\Pi_{OD}(\mathbf{k})$.

\textbf{Part 2: Equivalence of the Induced Trajectories -- }
It remains to show that $\Phi$ preserves the induced trajectory distribution. To characterize the system behavior explicitly, we define a trajectory $\mathcal{T}_T$ up to time $T$:
\[
\mathcal{T}_T = \left( s_0, \mathbf{o}_0, \mathbf{a}_0, s_1, \mathbf{o}_1, \mathbf{a}_1, \dots, s_T, \mathbf{o}_T, \mathbf{a}_T \right)
\]
We now derive the probability distribution $Pr(\mathcal{T}_T \mid \boldsymbol{\pi})$ induced by a fixed joint policy $\boldsymbol{\pi}$. By the chain rule of probability and the Markov property of the Dec-POMDP environment, the probability of a full trajectory factorizes as:

\begin{equation}
    Pr(\mathcal{T}_T \mid \boldsymbol{\pi}) = Pr(s_0) \prod_{t=0}^T \left(
    \underbrace{Pr(\mathbf{o}_t \mid s_t, \mathbf{a}_{t-1}) Pr(s_{t+1} \mid s_t, \mathbf{a}_t)}_{\text{Environment Dynamics}} \prod_{i=1}^n \underbrace{Pr(a_{i,t} \mid h_{i,t})}_{\text{Agent Policy}}
    \right)
\end{equation}


The \textbf{Environment Dynamics} terms, $P(s_t,\mathbf{a}_t)= Pr(s_{t+1} \mid s_t, \mathbf{a}_t)$ and $Pr(\mathbf{o}_t \mid s_t, \mathbf{a}_{t-1}) = O(s_t, \mathbf{a}_{t-1})$, are defined strictly by the system. By Assumption~\ref{assumption:no-leakage} (No Information Leakage), the observation process depends only on the current environment state and action and is unaffected by the agents' internal delay mechanisms or private buffers. Thus, the Environment Dynamics terms are perfectly invariant across both delay regimes.

\paragraph{Conclusion:}
Because the policy terms $Pr(a_{i,t} \mid h_{i,t}) = \pi_i(a_{i,t} \mid h_i^{(t, k_i)})$ are mathematically identical for every agent $i$ and time $t$ in both regimes, and the environment dynamics are identical, the total product $Pr(\mathcal{T}_T \mid \boldsymbol{\pi})$ evaluates to the exact same value under OD and AD. 

By induction on $t$, the stochastic processes generated by the AD and OD regimes are identical in distribution under all stated assumptions. Therefore, $\Phi$ is an isomorphism between policy spaces that preserves the induced joint trajectory measure.
\end{proof}

\begin{corollary}[Equivalence of Optimal Value]\label{corr:eq-optimal}
Let $\xi = (r_0, r_1, r_2, \dots)$ denote a reward sequence corresponding to a system trajectory (history) $\mathcal{T}$ of arbitrary length. Let $G(\xi)$ be the cumulative return of the trajectory, defined as any measurable function of the reward sequence (e.g., discounted sum $\sum \gamma^t r_t$, total reward $\sum_{t=0}^T r_t$, or average reward).
Let $J(\boldsymbol{\pi}) = \mathbb{E}_{\xi \sim \boldsymbol{\pi}}[G(\xi)]$ denote the expected return induced by the joint policy $\boldsymbol{\pi}$. 
Under the assumptions of Theorem \ref{theorem:main-equ}, for any joint policy $\boldsymbol{\pi} \in \Pi_{AD}(\mathbf{k})$:
\begin{equation}
J_{AD}(\boldsymbol{\pi}) = J_{OD}(\Phi(\boldsymbol{\pi}))
\end{equation}
Consequently, the maximum achievable return coincides under both regimes:
\begin{equation}
\sup_{\boldsymbol{\pi} \in \Pi_{AD}(\mathbf{k})} J(\boldsymbol{\pi}) =\sup_{\boldsymbol{\pi}' \in \Pi_{OD}(\mathbf{k})} J(\boldsymbol{\pi}')\triangleq J^*
\end{equation}
Moreover, a joint policy $\boldsymbol{\pi}^*$ is optimal in the AD regime if and only if its counterpart $\Phi(\boldsymbol{\pi}^*)$ is optimal in the OD regime.
\end{corollary}

\paragraph{Remarks:}
\begin{itemize}
    \item \textbf{Scope of Equivalence:} This equivalence holds strictly for fixed joint policies and completely characterizes the induced stochastic process. However, this result does \textit{not} claim that the AD and OD regimes remain symmetric during the learning phase, under adaptive policies, or under protocols where agents infer others' delays. As we will empirically demonstrate in Section 5, the practical learning dynamics under standard reinforcement learning algorithms exhibit fundamental asymmetries.
    \item \textbf{Structural Optimality:} Because the expected return $J(\boldsymbol{\pi})$ depends solely on the induced trajectory measure, this corollary is an immediate, structural consequence of Theorem \ref{theorem:main-equ}. The isomorphism inherently guarantees matching optimal values without requiring any additional, regime-specific optimality proofs.
\end{itemize}

\subsection{Mixed Delay Configuration}

\begin{corollary}[Equivalence of Mixed Delay Configurations]\label{corr:eq-mixed}
Let $\mathbf{\rho} = (\rho_1, \dots, \rho_n)$ be a \textbf{Type Configuration Vector}, where $\rho_i \in \{AD, OD\}$ denotes whether agent $i$ is subject to AD or OD.
Let $\Pi_{\mathbf{\rho}}(\mathbf{k})$ be the space of valid joint policies for the heterogeneous mixed-delay system. Under assumptions Assumption \ref{assumption:deterministic-buffer},\ref{assumption:dec-joint},\ref{assumption:no-leakage}, \ref{assumption:init}, the space $\Pi_{\mathbf{\rho}}(\mathbf{k})$ is globally isomorphic to the pure OD space $\Pi_{OD}(\mathbf{k})$.
\end{corollary}

\begin{proof}
By Assumption \ref{assumption:dec-joint}, the joint policy space decomposes into the Cartesian product of individual policy spaces:
\begin{equation}
    \Pi_{\mathbf{\rho}}(\mathbf{k}) = \bigotimes_{i=1}^n \Pi_{i, \rho_i}(k_i)
\end{equation}
From Theorem \ref{theorem:main-equ}, we established a bijective mapping $\phi_i: \Pi_{i, AD}(k_i) \to \Pi_{i, OD}(k_i)$ based on the equivalence of the observation-action history $h_i^{(t, k_i)}$.
For any agent $i$ where $\rho_i = AD$, we apply the mapping $\phi_i$. For any agent $j$ where $\rho_j = OD$, we apply the identity map $\text{id}_j$.
Construct the global bijection $\Psi$ as the tensor product of these individual mappings:
\begin{equation}
    \Psi = \bigotimes_{i=1}^n \psi_i, \quad \text{where } \psi_i = \begin{cases} \phi_i & \text{if } \rho_i = AD \\ \text{id}_i & \text{if } \rho_i = OD \end{cases}
\end{equation}
This mapping $\Psi: \Pi_{\mathbf{\rho}}(\mathbf{k}) \to \Pi_{OD}(\mathbf{k})$ is a bijection that preserves the induced state-action trajectories, confirming that any mixed-delay problem is reducible to an observation-action pure OD problem.
\end{proof}

\begin{corollary}[Additivity of Simultaneous Delays]\label{corr:eq-add}
Consider a generalized configuration where every agent $i$ can be simultaneously subject to both AD $k_{i,a}$ and OD $k_{i,o}$. We denote the OD configuration vector $\mathbf{k_o}$ and the AD configuration vector $\mathbf{k_a}$. The joint policy space of this system is isomorphic to that of a pure OD system with effective delay vector $\mathbf{K} = \mathbf{k}_a + \mathbf{k}_o$.
\end{corollary}

\begin{proof}
We examine the information available to agent $i$ for an action $a_{i,t}$ effective at time $t$.
\begin{enumerate}
    \item \textbf{Decision Time Shift:} Due to AD $k_{i,a}$, the action must be committed at decision time $\tau = t - k_{i,a}$.
    \item \textbf{Observation Window Shift:} At decision time $\tau$, due to OD $k_{i,o}$, the most recent observation available is $o_{i, \tau - k_{i,o}}$.
\end{enumerate}
Substituting $\tau$, the observation cutoff index becomes:
\begin{equation}
    (\tau - k_{i,o}) = (t - k_{i,a}) - k_{i,o} = t - (k_{i,a} + k_{i,o})
\end{equation}
Since the agent retains internal knowledge of its past decisions $a_{i, 0:t-1}$, the effective information set is exactly the observation-action history $h_i^{(t, K_i)}$ where $K_i = k_{i,a} + k_{i,o}$.

Consequently, the mapping $\phi_i$ from Theorem \ref{theorem:main-equ} applies with delay length $K_i$, establishing the isomorphism to the pure OD space.
\end{proof}

The combination of Theorem \ref{theorem:main-equ} and Corollaries 2-3 implies that any decentralized system characterized by arbitrary combinations of fixed network latencies (observation transmission) and actuation lags (mechanical or processing) can be formally modeled as a standard Dec-POMDP with pure ODs. This allows for the unified application of solution methods designed for OD environments (and conversely for solutions designed for AD).

\subsection{Sufficient Statistics in Transition-Independent Settings}

As established in the previous section, the formal equivalence between AD and OD holds when considering the complete observation-action history for each agent. However, for this equivalence to yield a tractable solution method, we must identify a compact sufficient statistic. In the general multi-agent case, an agent cannot predict the evolution of the global state during the delay interval because it lacks information about the simultaneous actions of teammates. Therefore, we focus on \textbf{Transition Independent (TI)} settings with \textbf{Full Local Observability}.

\begin{definition}[TI-Dec-MDP]
We consider a Transition Independent Dec-MDP where:
\begin{enumerate}
    \item \textbf{Full Local Observability:} The global state space decomposes as $\mathcal{S} = \times_{i\in\mathcal{N}} \mathcal{S}_i$, the product of local state spaces. At $t$ a state is the vector $s_t=\{s_{i,t}\}_{i=1,...n}$. Each agent $i$ perfectly observes its own local state component: 
    \[
    o_{i,t} = s_{i,t}
    \]
    \item \textbf{Independence:} The dynamics factorize locally, defining a local transition function $P_i$ for each local state space $\mathcal{S}_i$ and local action space $A_i$ such as: 
    \[
    P(\mathbf{s}' \mid \mathbf{s}, \mathbf{a}) = \prod_{i=1}^n P_i(s_i' \mid s_i, a_i)
    \]
\end{enumerate}
\end{definition}

\begin{definition}[Local Augmented Information State]
Let $k_i$ be the delay length for agent $i$. The \textit{Local Augmented State} contains the \textbf{last observed state} and the \textbf{$k_i$-step sequence of last decided actions}. We define it based on the decision time:
\begin{itemize}
    \item \textbf{In OD (Decision at effective time $t$):} The agent observes the delayed state and remembers the actions that have been executed since:
    \[ \tilde{s}_{i,t}^{OD} \triangleq \langle s_{i, t-k_i}, \quad \mathbf{a}_{i, t-k_i : t-1} \rangle \]
    \item \textbf{In AD (Decision at time $\tau$, effective at $t = \tau + k_i$):} The agent observes the current state and remembers the past decisions currently pending in the buffer:
    \[ \tilde{s}_{i,\tau}^{AD} \triangleq \langle s_{i, \tau}, \quad \mathbf{a}_{i, \tau-k_i : \tau-1} \rangle \]
\end{itemize}

We denote the joint augmented state-space $\tilde{\mathcal{S}}$ and its local subset for an agent $i$, $\tilde{\mathcal{S}_i}$. 
\end{definition}

Then, we derive a practically useful result by giving a sufficient local augmented state for learning an optimal joint policy in TI Dec-MDP.

\begin{theorem}[Sufficiency of the Augmented State under Transition Independence]\label{theorem:independent-sufficient}
For any TI-Dec-MDP with delay configuration $\mathbf{k}$, the local augmented state $\tilde{s}_{i}$ is a sufficient statistic for optimal local control. There exists an optimal joint policy $\boldsymbol{\pi}^*$ composed of independent local policies that depend strictly on the local augmented state:
\begin{equation}
    \pi_i : \tilde{\mathcal{S}}_i \to \Delta(A_i) 
\end{equation}
\end{theorem}

\begin{proof}
To prove sufficiency for optimal local control, we must show that $\tilde{s}_{i}$ contains all necessary and sufficient information to predict the local state distribution at the moment the chosen action takes effect. By the definition of a TI-Dec-MDP, local state transitions depend exclusively on local states and local action. The actions of other agents do not interfere.

\paragraph{Case 1: Observation Delay (OD)}
At effective time $t$, agent $i$ must select an action $a_{i,t}$ that takes effect immediately. The available information is $\tilde{s}_{i,t}^{OD} = \langle s_{i, t-k_i}, \mathbf{a}_{i, t-k_i : t-1} \rangle$. Because the environment is transition independent, the evolution of the local state from step $t-k_i$ to $t$ is strictly a function of the state at $t-k_i$ and the sequence of actions ($\mathbf{a}_{i, t-k_i : t-1}$) executed during the interval $[t-k_i,t]$. 
By the Markov property, the true current state $s_{i,t}$ is conditionally independent of all older history:
\[
Pr(s_{i,t} \mid \text{history}_t) = Pr(s_{i,t} \mid s_{i, t-k_i}, \mathbf{a}_{i, t-k_i : t-1}) = Pr(s_{i,t} \mid \tilde{s}_{i,t}^{OD})
\]

\paragraph{Case 2: Action Delay (AD)}
At decision time $\tau$, agent $i$ must select an action that will take effect at future time $t = \tau + k_i$. To evaluate the expected value of this decision, the agent must project the state to the effective time $t$. The available information is $\tilde{s}_{i,\tau}^{AD} = \langle s_{i, \tau}, \mathbf{a}_{i, \tau-k_i : \tau-1} \rangle$. 
Thus, the actions residing in the augmented state memory $\tilde{s}_{i,\tau}$ (decided between $\tau-k_i$ and $\tau-1$) are exactly the actions that will execute during the interval $[\tau, \tau+k_i-1]$ by Assumption \ref{assumption:deterministic-buffer}. By transition independence, the state at execution time $s_{i, \tau+k_i}$ depends strictly on the currently observed state $s_{i,\tau} = o_{i,\tau}$ and these pending actions. Thus:
\[
Pr(s_{i, \tau+k_i} \mid \text{history}_\tau) = Pr(s_{i, \tau+k_i} \mid s_{i, \tau}, \mathbf{a}_{i, \tau-k_i : \tau-1}) = Pr(s_{i, \tau+k_i} \mid \tilde{s}_{i,\tau}^{AD})
\]

\paragraph{Conclusion}
In both regimes, the local augmented state provides the observation and the $k_i$-step sequence of actions necessary to project the environment dynamics forward to the moment where the action takes effect. Consequently, $\tilde{s}_{i,t}$ acts as a sufficient statistic for the local state, allowing the exact reduction of the delayed TI-Dec-MDP to a standard MDP over the augmented state space $\tilde{\mathcal{S}}_i$.
\end{proof}

This matches the state-augmentation method proposed by \cite{katsikopoulos2003markov} for the single-agent setting, and effectively provides a tractable method in multi-agent settings for transition-independent environments. This method ensures the reduction to MDP only in a fully observable and transition-independent setting. We will show empirically in the next section that the learning can be significantly different when using this augmented state in environments with coupled transitions.

\section{Practical Learning Gap}\label{section:practical-gap}

We derived results regarding the formal equivalence of OD and AD problems for decentralized joint policies under classical assumptions of multi-agent systems. These results are formally interesting regarding the theory of Dec-POMDPs, as they extend classical results from single-agent reinforcement learning literature to the multi-agent case. The isomorphism gives a strict trajectory equivalence for joint policy evaluation, but it does not give any results regarding learning dynamics in practice. In this section, we discuss the differences between OD and AD in learning problems. First, we highlight that the AD regime introduces an additional reward delay as compared to OD, which consequently induces credit assignment challenge in TD-Learning requiring an additional realignment mechanism to retrieve equivalent learning dynamics. Secondly, we empirically show on a simple example how equivalence fail when using local augmented state described in Theorem \ref{theorem:independent-sufficient} in transition-dependent environment. Then, we briefly study the implications of the initialization assumption, highlighting its necessity and showing how it breaks symmetry of delay regimes. Finally, despite identified asymmetries in learning dynamics, we empirically showcase that even in complex environments, zero-shot transfer learning can be performed between OD and AD regimes, illustrating a practical use case for MARL on real delayed systems.

\subsection{Experimental Setup}

We evaluate our theoretical findings using a cooperative Multi-Agent Grid World navigation task, a foundational benchmark domain for sequential decision-making \citep{sutton2018reinforcement, oliehoek2016concise}. The environment is a grid with two agents. The objective for the agents is to simultaneously occupy the two targets they are assigned to end the episode and receive a positive reward. Agents observe their position on the grid, eventually with a delay and at each step, agent $i$ can choose action as $a_i(t) \in \{\text{up},\text{down},\text{left}, \text{right} \}$. Each step gives a negative reward $r(t) = -1$ and reaching both target ends the episode and gives $r(t) = 10$. In addition, to introduce coordination between agents, we place a wall in the middle of the grid with a single corridor (see Figure \ref{fig:env-illustration}). 

The environment can then be instantiated in two different settings:
\begin{itemize}
    \item \textbf{Transition-Independent:} Agents act fully independently. They can occupy the same position on the grid. In this setting, multi-agent coupling only occurs through the shared reward.
    \item \textbf{Coupled-Transition:} Agents cannot occupy the same position nor be within their respective collision radius (set to $1$ in our experiments). When agents attempt to enter each other's collision radius, the move cannot be done. In this setting, each agent's transition model depends on others' actions. In this setting, agents cannot simultaneously cross the corridor.
\end{itemize}

\begin{figure}[htp]
    \centering
    \includegraphics[width=0.5\linewidth]{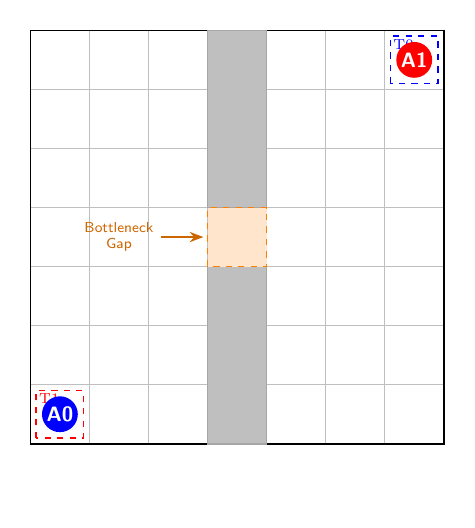}
    \caption{\textbf{Illustration of the Cooperative Grid World Environment at $t=0$.}}
    \label{fig:env-illustration}
\end{figure}

Detailed algorithms for the environment's transition, collision mechanism, and delay wrapper are given in Appendix \ref{suppl:exp-setup}.

\subsection{Explaining Learning Asymmetry in TD-Learning}\label{subsec:td-learning}

To highlight differences that exist when learning under the two regimes, we look at a standard method in RL: Temporal Difference (TD) learning. This is at the basis of Q-Learning and SARSA algorithms \cite{sutton2018reinforcement}. For simplicity of derivation, we consider a single-agent MDP here but the results apply to dec-POMDP. For Q-Learning, in a standard, undelayed MDP, the TD error is calculated as:
\begin{equation}\label{eq:TD-update}
    \delta_t = \underbrace{r_{t+1} + \gamma \max_{a'} Q(s_{t+1}, a')}_{\text{Target}} - \underbrace{Q(s_t, a_t)}_{\text{Estimate}}
\end{equation} 
with $Q$ the state-action value function, representing the expected cumulative discounted reward of taking action $a_t$ in state $s_t$ and following the optimal policy thereafter.

This update relies on direct temporal causality: the state and action pair $(s_t, a_t)$ that induced a transition is immediately associated with the resulting reward $r_{t+1}$. By explicitly writing the Q-value updates for both delay regimes, one can clearly see how AD breaks this causality. We explicitly derive the differences in TD-updates with OD and AD regimes.\\

Let $s_t$ be the state of the environment at time $t$. To clarify the temporal dynamics, we introduce a notation separating decision time and effective time: let $a^i_{(\tau,t)}$ denote agent $i$'s local action, \textit{decided} at time $\tau$ and \textit{effective} at time $t$. The joint action \textit{effective} at $t$ is denoted simply as $a_t$. 

Based on the local minimal augmented state established in Theorem \ref{theorem:independent-sufficient}:
\begin{itemize}
    \item \textbf{In OD:} $\tilde{s}_t = (o_{t-k}, a^i_{(t-k,t-k)}, \dots, a^i_{(t-1,t-1)})$. The observation reflects a past state, and the buffer contains the history of past actions that were decided and effective instantaneously up to $t-1$.
    \item \textbf{In AD:} $\tilde{s}_t = (o_{t}, a^i_{(t-k,t)}, \dots, a^i_{(t-1,t+k-1)})$. The observation reflects the current state, and the action buffer contains previously \textit{decided} actions that will take effect sequentially from step $t$ to $t+k-1$.
\end{itemize}

Let $\alpha$ be the learning rate. Using this explicit temporal notation, we can rewrite the local TD-updates for each regime:

\paragraph{OD - Causality with Immediate Reward Collection:}
The standard TD update correctly pairs the cause with its immediate effect:
\begin{equation}\label{eq:q-learning-od}
    Q(\tilde{s}_t, a^i_{(t,t)}) \leftarrow Q(\tilde{s}_t, a^i_{(t,t)}) + \alpha \left[ r_{t+1} + \gamma \max_{a'} Q(\tilde{s}_{t+1}, a') - Q(\tilde{s}_t, a^i_{(t,t)}) \right]
\end{equation}

\paragraph{AD Naive - Breaking Causality:}
If a standard TD algorithm is applied naively to an environment with an instantaneous reward, it erroneously updates the value of an action that will take effect in the future using a current reward. At decision time $t$, agent $i$ decides action $a^i_{(t,t+k)}$ using observation $\tilde{s}_t$. However, the environment transitions based on $\mathbf{a}_{(t-k,t)}$ (whose local component was decided at $t-k$). The naive TD update is:
\begin{equation}\label{eq:q-learning-naive}
    Q(\tilde{s}_t, a^i_{(t,t+k)}) \leftarrow Q(\tilde{s}_t, a^i_{(t,t+k)}) + \alpha \left[ r_{t+1} + \gamma \max_{a'} Q(\tilde{s}_{t+1}, a') - Q(\tilde{s}_t, a^i_{(t,t+k)}) \right]
\end{equation}

Here, $a^i_{(t,t+k)}$ is a component of the joint action $\mathbf{a}_{t,t+k}$ and will not be effective until $t+k$, at which point it will generate reward $r_{t+k+1}$. Therefore, the true causal reward for the decision made at $t$ will not be observed until step $t+k+1$. Structurally, this is identical to an environment where actions are immediately effective, but their corresponding reward is delayed by $k$ steps.

\paragraph{AD Realigned - Restoring Causality:}
To recover the causal chain within the augmented state space, the TD update for a decision made at $t$ must be suspended until it becomes effective. By buffering the agent's transition tuple $(\tilde{s}_t, a^i_{(t,t+k)}, \tilde{s}_{t+1})$ for $k$ steps, we arrive at effective step $t+k$. At this step, the decision $a^i_{(t,t+k)}$ takes effect on the environment, yielding its true causal reward $r_{t+k+1}$. \\
Shifting the time index backward for clarity: at any current effective step $t$, the action currently taking effect is $a_t$, whose $i$-th component is $a^i_{(t-k,t)}$, generating $r_{t+1}$. We retrieve the historical, buffered transition $(\tilde{s}_{t-k}, a^i_{(t-k,t)}, \tilde{s}_{t-k+1})$ and apply the delayed update:
\begin{equation}\label{eq:q-learning-realigned}
    \begin{split}
        Q(\tilde{s}_{t-k}, a^i_{(t-k,t)}) \leftarrow & ~ Q(\tilde{s}_{t-k}, a^i_{(t-k,t)}) \\
        & + \alpha \left[ r_{t+1} + \gamma \max_{a'} Q(\tilde{s}_{t-k+1}, a') - Q(\tilde{s}_{t-k}, a^i_{(t-k,t)}) \right]
    \end{split}
\end{equation}

This realigned update perfectly mirrors the causal structure of the OD equation. It pairs the augmented state $\tilde{s}_{t-k}$ and the corresponding decision $a^i_{(t-k,t)}$ with the exact subsequent augmented state $\tilde{s}_{t-k+1}$ and its direct consequence $r_{t+1}$, effectively eliminating the credit assignment error. \\

Finally, we derive the following proposition:
\begin{proposition}[Reward Delay in TD-Learning with Action Delay]\label{prop:td-learning-asymetry}
    Let an agent interact with an AD environment with a delay $k$. A standard 1-step Temporal Difference (TD) update (Equation \ref{eq:TD-update}) applied naively at decision time $t$ introduces a $k$-step causal mismatch (Equation \ref{eq:q-learning-naive}), rendering the algorithmic credit assignment equivalent, in terms of the sequence of TD updates, to an OD regime coupled with a strict $k$-step reward collection delay. Furthermore, buffering the transition tuple and realigning the TD update to the effective time (Equation \ref{eq:q-learning-realigned}) completely resolves this causal mismatch, restoring the exact state-action-reward credit assignment of the OD regime (Equation \ref{eq:q-learning-od}).
\end{proposition}

Figure \ref{fig:alignment} illustrates how we can practically retrieve OD's learning efficiency by applying the described effective-time realignment.

\begin{figure}[htp]
    \centering
    \includegraphics[width=0.66\linewidth]{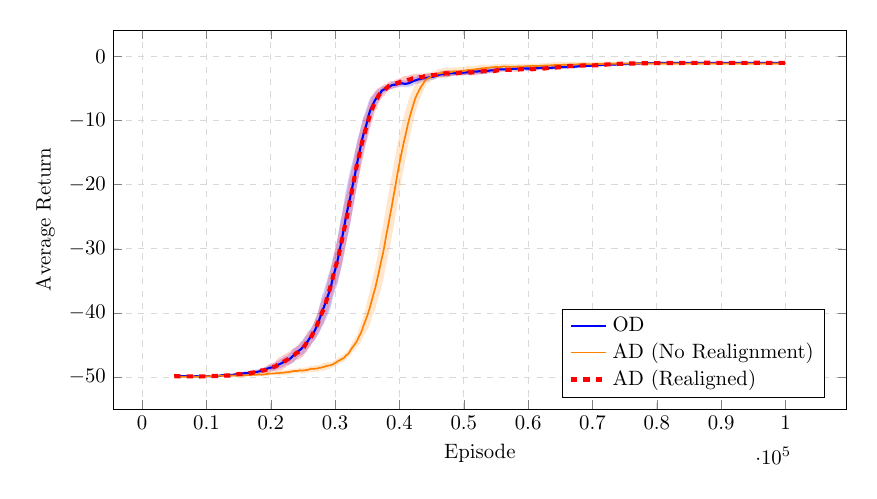}
    \caption{\textbf{Effective-Time Alignment retrieving the OD learning regime:} By buffering observation-action pairs for each agent and waiting for the reward  induced by an action, we realign learning under AD to the OD regime. Importantly, this showcases that the difficulty of performing TD-Learning in the AD regime is solely due to the reward collection delay induced by the problem as compared to the OD regime.}
    \label{fig:alignment}
\end{figure}

 To summarize and provide a better intuition: while AD delays the execution of control and OD delays the perception of the state, from the perspective of a Temporal Difference (TD) learning algorithm, AD is structurally equivalent to OD combined with a strict $k$-step reward collection delay. A direct consequence of Proposition \ref{prop:td-learning-asymetry} is that restoring causality in TD-Learning under AD requires more memory than under OD. Indeed, in AD, one needs to store the $k$-sequence of observation-action pairs to properly update the Q-table when the reward generated by a decision is received. Thus, it requires a memory overhead of $\mathcal{O}(k)$. Under OD, this reward is immediately available to update the Q-table. It should be noted that any additional reward delay in the system would create or increase the credit assignment issue for both regimes and is additive to the AD's specific reward delay. We use the realigned TD Learning methods for AD in the rest of the experiments.

 \paragraph{Remarks:} \cite{katsikopoulos2003markov} implicitly treat this asymmetry in Q-Learning by assuming a reward delay larger than the $OD$ (or $AD$). Doing that, they are actually introducing a reward delay to the $OD$ regime to compensate for the pre-existing delay supposed by the $AD$ regime and described in this section. Our algorithmic approach does the inverse, effectively realigning AD regimes through buffering without assuming additional reward delay.

 \paragraph{On-Policy Algorithms:} 
 
 While the effective-time realignment restores the causal chain for off-policy algorithms like Q-Learning, it introduces a structural "policy drift" in on-policy algorithms like SARSA. If we apply the realignment to SARSA at effective time $t$, the TD target requires the action taking effect at $t+1$, which was buffered at $t-k+1$. The update becomes:
 $$Q_t(\tilde{s}_{t-k}, a^i_{(t-k,t)}) \leftarrow Q_t(\tilde{s}_{t-k}, a^i_{(t-k,t)}) + \alpha \left[ r_{t+1} + \gamma Q_t(\tilde{s}_{t-k+1}, \underbrace{a^i_{(t-k+1,t+1)}}_{\sim \pi_{Q_{t-k+1}}}) - Q_t(\tilde{s}_{t-k}, a^i_{(t-k,t)}) \right]$$
 Here, $a^i_{(t-k+1,t+1)}$ was sampled using a past policy $\pi_{t-k+1}$, but is evaluated using the current value function $Q_t$. This forces the on-policy update to consume off-policy data without correction, introducing a learning bias that scales with the delay $k$. In tabular settings with sparse updates, this drift is negligible, but it becomes a significant bottleneck for function approximators. This is of interest for future work.

\subsection{The Role of Reward Delay: Separating Policy Admissibility from Credit Assignment}

Our theoretical equivalence between delay regimes fundamentally relies on defining local policies as functions of the local state-action history (Definition \ref{def:local-policy}). Because the proof of equivalence relies exclusively on the bijection between the sets of admissible policies, ensuring any effective action is chosen using the exact same information set, the learning dynamics and reward collection mechanisms are intentionally decoupled from the theoretical proof. Consequently, our baseline Observation Delay (OD) regime assumes immediate reward collection.

In real systems, however, reward signals are often intrinsically coupled with state observations. This leads to the standard assumption in the literature that an observation delay inherently induces an equivalent reward delay ($k_{reward} \geq k_{obs}$). This ensures the reward $R(s_{t-k},\mathbf{a}_{t-k})$ only becomes available when the observation induced by $s_{t-k}$ reaches the agent \citep{katsikopoulos2003markov}. While we acknowledge this is a standard formulation for deployed systems, we argue that learning a policy under OD with immediate rewards has profound practicality in modern Reinforcement Learning, particularly in Sim-to-Real paradigms. In such setups, a policy is trained within a simulator where rewards can be computed and delivered instantaneously, even while the simulated observation stream is artificially delayed to match the effective OD of the targeted real-world deployment.

Furthermore, introducing a reward delay does not invalidate the structural equivalence of the delay regimes, provided the information available for action selection remains clearly defined. As discussed in the literature \citep{liotet2022delays}, reward delay is fundamentally a credit assignment problem, and therefore a matter of pairing the correct transition tuple with its corresponding induced reward. Whether an agent learns a policy with immediate rewards or delayed rewards only induces credit assignment challenge, which can be mitigated algorithmically (e.g., via transition buffering) as demonstrated in Section \ref{subsec:td-learning}.

To empirically validate that reward delay does not alter the underlying equivalence, we introduce a strict reward delay $k_{reward} = k_{obs} = k$ to the OD environment. The reward associated with a state is now received exactly when that state is observed. We evaluate three distinct algorithmic treatments of this delayed reward stream:

\begin{itemize}
    \item \textbf{Naive (Truncated):} The agent updates the Q-table using the current delayed observation and action whenever a delayed reward is received. There is no buffering mechanism. Consequently, at the episodic boundary, the final $k$ rewards (including the terminal goal reward) are truncated by the environment's termination and are never processed by the learning agent, leading to severe sample inefficiency and blindness to the goal.
    \item \textbf{Flushed (Isomorphic Construction):} To maintain a symmetric learning trajectory, the agent processes the delayed rewards but forces a subjective continuation at the end of an episode. Upon reaching the final effective step $T$, the agent stops acting but extracts the remaining unobserved states ($\tilde{s}_{T-k+1} \dots \tilde{s}_T$) from the environment's buffer. By pairing these true states and the remaining delayed rewards, this flush mechanism allows the OD agent to reconstruct the acausal credit assignment tuples of the Naive AD regime.
    \item \textbf{Reward Aligned (Causal Restoration):} Similar to the effective-time realignment used for AD, the agent maintains an internal transition buffer. Subjective state-action tuples are held for $k$ steps and only updated when their true causally induced reward arrives from the delayed pipeline. To ensure no data is lost, this buffer is fully flushed at the end of the episode, perfectly restoring the optimal learning trajectory of the baseline OD regime.
\end{itemize}

\begin{figure}[htp]
    \centering
    \includegraphics[width=0.66\linewidth]{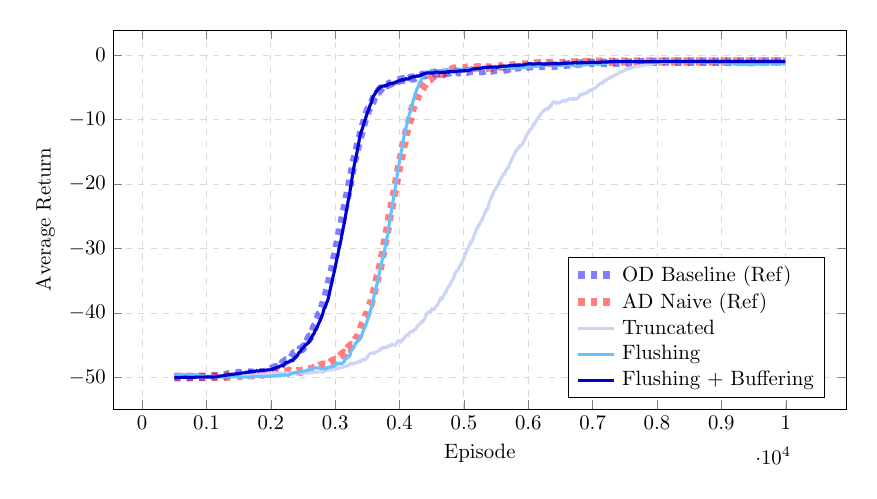}
    \caption{\textbf{Empirical validation of algorithmic isomorphism under delayed reward collection.} Learning trajectories in an OD environment subjected to a strict $k$-step reward delay. The naive \textit{Truncated} approach suffers severe degradation due to lost terminal goal rewards. The \textit{Flushing} mechanism processes unseen boundary states to perfectly reconstruct the acausal learning trajectory of Naive Action Delay (AD). Conversely, the \textit{Flushing + Buffering} approach achieves perfect causal restoration, buffering transitions to match their delayed rewards and successfully recovering the optimal performance of the OD Baseline.}
    \label{fig:od-delays}
\end{figure}

\paragraph{Remarks on Episodic Boundaries:} 
Interestingly, the necessity of flushing the remaining observations and rewards at the end of a finite-horizon episode acts as the terminal counterpart to the \textit{Warm Start} initialization assumed for the AD regime. However, a fundamental distinction exists in their theoretical purpose in this work: the \textit{Warm Start} initialization at $t=0$ is a formal prerequisite necessary to construct the bijection of admissible policies between regimes, whereas the terminal flush at $t=T$ is strictly an algorithmic intervention necessary to rescue uncollected rewards and reconstruct the causal learning trajectory.

\subsection{Transition Independence and Sufficient Statistics}

The minimal augmented state $\tilde{s}_i$ derived in Theorem \ref{theorem:independent-sufficient} is theoretically sufficient only under Transition Independence and full local observability. Under these strict conditions, standard RL algorithms can be applied directly, yielding symmetric learning dynamics between the OD and AD regimes with realignment (Figure \ref{fig:sub_indep}). However, this symmetry breaks down when coupling is introduced.

\begin{figure}[htp]
    \centering
    \begin{subfigure}[b]{0.45\textwidth}
        \centering
        \includegraphics[width=\linewidth]{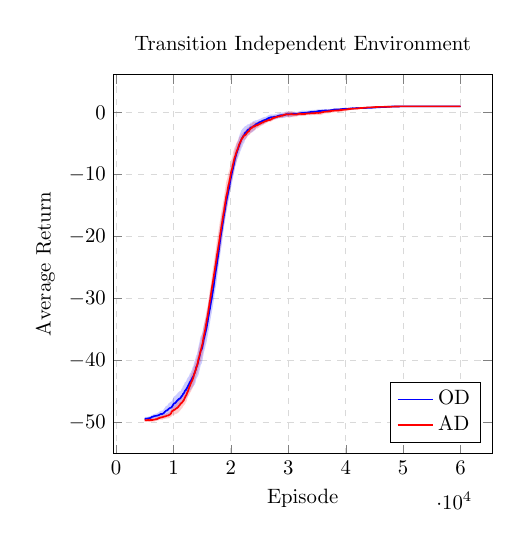}
        \caption{}
        \label{fig:sub_indep}
    \end{subfigure}
    \hfill
    \begin{subfigure}[b]{0.45\textwidth}
        \centering
        \includegraphics[width=\linewidth]{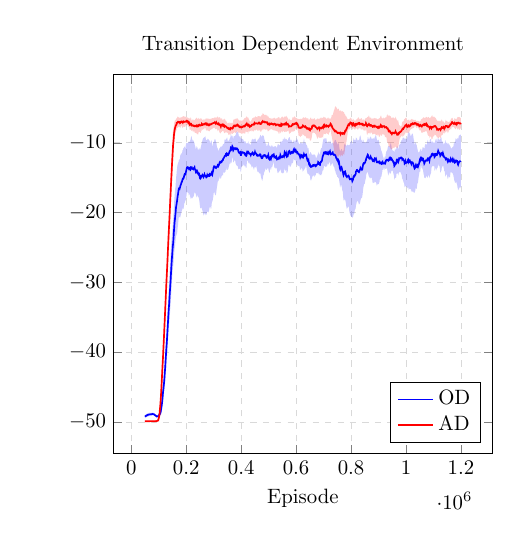}
        \caption{}
        \label{fig:sub_coupling}
    \end{subfigure}
    \caption{\textbf{Limits of Minimal Augmented State in Transition-Dependent Systems.} Learning curves of Tabular Q-Learning under OD and AD using the local augmented state $\tilde{s}_i = (o_{i, t-k}, a_{i, t-k:t-1})$, the delay vector is $k=[0,3]$. In (\ref{fig:sub_indep}), the environment is transition-independent. In (\ref{fig:sub_coupling}), agents can collide, introducing coupling. The grid is $(6\times 6)$, and the collision radius is of size $1$. Learning curves are averaged over 10 trajectories.}
    \label{fig:coupling}
\end{figure}

As shown in Figure \ref{fig:sub_coupling}, introducing inter-agent collisions breaks transition independence, causing the learning dynamics of OD and AD to fundamentally diverge. In this specific setting, AD converges to a better joint policy, whereas learning in the OD regime is noisier and has poorer performance.

This divergence is explained by the limits of our theoretical assumptions. Transition independence ensures that the local augmented state $\tilde{s}$ preserves the Markov property. When coupling is introduced, an agent's local transition becomes dependent on the concurrent, unobservable actions of its teammates. Consequently, $\tilde{s}$ loses its Markovian property, and the resulting asymmetry arises entirely from how the respective learning algorithms process these non-Markovian transitions under different delay structures. 

It is important to note that the poor performance of the OD regime here is specific to this environment's structure; other coupled problems may induce the opposite result. Overall, coupling breaks the sufficiency of the minimal augmented state and induces a structural asymmetry that makes the learning process fundamentally different across regimes. Therefore, to guarantee the theoretical equivalence of optimal policies, the remainder of our Grid World experiments use the Transition-Independent setting.

\subsection{Warm Start vs Cold Start: Implication of the Initialization Assumption}

Assumption \ref{assumption:init} is theoretically useful as it ensures that any divergence between the regimes is attributable solely to information delay dynamics, rather than initialization artifacts. However, it is the most restrictive in practice. In the rest of the paper, we call \textit{Warm Start} the initialization for an agent that respects Assumption \ref{assumption:init}. Conversely, we use \textit{Cold Start} to refer to a system where no action is taking effect on the environment during the first few timesteps, and hence the agent is subject to dead time for the span of the delay. Practically, a \textit{Warm Start} setting would allow the agent to plan its first actions ahead according to its policy based on arbitrary initial information. In many systems, however, AD would enforce a \textit{Cold Start} default initialization behavior within the delay time (e.g. agent is waiting). In this context, and with equivalent delay, it is easy to see that the OD regime provides a clear advantage for the agent as it can act on the system during delay time, whereas the agent subject to AD should wait for the action to be effective while observing the system. Therefore, we derive the following proposition:

\begin{proposition}[Strict Inclusion under Cold Start]
\label{prop:cold-start-inclusion}
Consider a relaxation of Assumption \ref{assumption:init} where the AD regime is subject to a "Cold Start" constraint.
Under this constraint, the set of admissible joint policies under AD $\Pi_{AD}(\mathbf{k})$ is a strict subset of those admissible under OD $\Pi_{OD}(\mathbf{k})$:
\begin{equation}
    \Pi_{AD}(\mathbf{k}) \subset \Pi_{OD}(\mathbf{k})
\end{equation}
\end{proposition}

\begin{proof}
The proof follows by establishing inclusion and strictness:

\textbf{1. Inclusion ($\subseteq$):} 
Let $\boldsymbol{\pi}_{AD} \in \Pi_{AD}(\mathbf{k})$ be any policy subject to the \textit{Cold Start} constraint. This implies that for all $t < k_i$, the effective action is $a_{i,t} = a_{\emptyset}$, corresponding to a default \textit{noop} action.
An agent in the OD regime can strictly emulate this behavior. We construct a valid policy $\boldsymbol{\pi}_{OD} \in \Pi_{OD}(\mathbf{k})$ such that:
\[
\pi_{OD}(a_{i,t} \mid h_t) = 
\begin{cases} 
\delta(a_{\emptyset}) & \text{if } t < k_i \\
\Phi(\pi_{AD})(a_{i,t} \mid h_t) & \text{if } t \ge k_i 
\end{cases}
\]
where $\delta(a_{\emptyset})$ denotes the Dirac measure, assigning a probability of 1 to the default action $a_{\emptyset}$.
Assuming the OD agent is capable of voluntarily ``waiting'' (selecting $a_{\emptyset}$), it can perfectly reproduce the trajectory distribution of the \textit{Cold Start} AD agent.

\textbf{2. Strictness ($\neq$):} 
Consider an OD policy $\boldsymbol{\pi}'_{OD}$ that selects a non-null action $a^* \neq a_{\emptyset}$ at $t=0$ with probability $1$. 
This generates a trajectory where $a_{i,0} = a^*$.
No policy in $\Pi_{AD}(\mathbf{k})$ subject to the Cold Start constraint can generate this trajectory, as the AD agent is forced to have $a_{\emptyset}$ at $t=0$.
Thus, there exist realizable outcomes in OD that are impossible in AD under this constraint.
\end{proof}

The superset argument was already identified in the li‡terature \citep{wu2024boosting} for the single-agent case. It justifies studying OD as a more general problem than AD. However, to the best of our knowledge, this had not been formally stated yet. Hence, Proposition \ref{prop:cold-start-inclusion} is closing this gap.\\
 
\textbf{Remark on Infinite Horizon:} The equivalence and reduction stated in \cite{katsikopoulos2003markov} in the fully observable single-agent case rely on an infinite-horizon assumption. They implicitly assume a \textit{Warm Start} initialization for their reduction.  Our framework is more formal, by enforcing consistent initialization (Assumption \ref{assumption:init}), we establish exact trajectory isomorphism for any finite horizon $T$. If we relax Assumption \ref{assumption:init} (i.e., operating under a \textit{Cold Start} in AD), strict trajectory equivalence is broken. To illustrate the structural superiority of the OD regime over a Cold Start AD regime in a finite horizon, we apply the realigned TD learning (\ref{eq:q-learning-realigned}) to the Grid world navigation problem with OD, AD with \textit{Warm Start} and AD with \textit{Cold Start}. Figure \ref{fig:rewards} effectively shows that OD and \textit{Warm Start} AD achieve strictly equivalent and superior returns compared to the AD regime with a \textit{Cold Start}.

\begin{figure}[htp]
    \centering
    \includegraphics[width=0.7\linewidth]{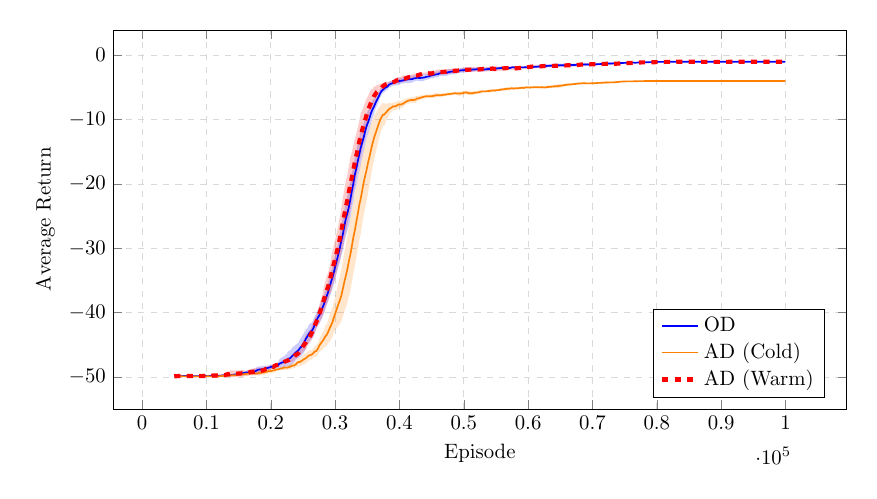}
    \caption{\textbf{Relaxing Initialization Assumption:} Average return with delay $k=[0,3]$ achieved by Q-Learning with augmented state $\tilde{s}_i$. The \textbf{AD (Warm Start)} configuration outperforms \textbf{AD (Cold Start)} and matches \textbf{OD} final performance. In OD and AD (Warm Start), the final average return is $-1$ and in AD (Cold Start), it is $-4$, corresponding to the $3$ steps of dead-time of one of the agent. We consider an open-grid of size $(7\times 7)$.}
    \label{fig:rewards}
\end{figure}

\subsection{Zero-shot Policy Transfer}

In previous sections, we demonstrated that OD and AD exhibit fundamentally asymmetric learning dynamics. Specifically, causal credit-assignment errors and strict initialization constraints make the AD regime practically harder to optimize. Additionally, coupling introduces non-Markovian transitions from the perspective of agents observing a minimal augmented state, limiting the applicability of such methods in partially observable and multi-agent settings.  However, despite these learning asymmetries and the theoretical reliance on Transition Independence, the structural isomorphism of the joint-policy spaces remains highly valuable for complex tasks. In practice, leveraging the local augmented state to perform zero-shot policy transfer from OD to AD emerges as a pragmatic solution to bypass these optimization bottlenecks entirely. 

Consequently, a major practical implication of our theoretical equivalence is that it enables zero-shot policy transfer from the OD regime to the AD regime, and conversely. Theoretically, this transfer is guaranteed as long as the problem respects our formal assumptions and the policy maintains the same input information structure. While these conditions are somewhat restrictive, this \textit{Zero-Shot} transfer argument remains highly relevant in the context of real systems applications. Most algorithms available in the literature are specifically designed for one or the other delay regime. Under our framework, for example, one can train a decentralized policy on a simulated environment that models OD using an efficient algorithm for the OD regime, and deploy the policy zero-shot on a fleet of agents in a real system constrained by AD.

Figure \ref{fig:zero-shot} shows zero-shot policy transfer from the OD regime (superset) to the AD regime on a more advanced cooperative environment: the Multiwalker \citep{gupta2017cooperative}\footnote{\url{https://pettingzoo.farama.org/environments/sisl/multiwalker/}}. In this environment, agents are bipedal robots trying to cooperatively transport a package placed on top of them through the map. They observe internal properties and sensor information. Robots falling or dropping the package induce a large penalty, and a reward is obtained for moving the package towards the right. This environment is highly challenging as the action space is a 4-dimensional continuous vector. We introduce delays into the environment through a custom delay wrapper that respects our baseline delay mechanics.

For each delay value, we trained an Actor-Critic policy with a Proximal Policy Optimization (PPO) algorithm \citep{schulman2017proximal} on an environment with OD. Since agents are identical, we use parameter sharing; that is, we train a policy that is shared among agents. This preserves decentralized execution but leverages experience from homogeneous agents. This has proven to be more efficient in complex cooperative multi-agent environments with homogeneous agents \citep{de2020independent}. Agents have access to their Local Augmented State as defined in Theorem \ref{theorem:independent-sufficient}. 

Obviously, such an environment cannot be considered transition-independent. However, the local augmented state offers a tractable solution in coupled complex environments, even though it may not be theoretically optimal compared to using the full Observation-Action History (which becomes intractable due to its high dimension). We evaluate the performance of the learned policy in the same environment with AD by measuring an average return over $300$ episodes. We observe that learned policies in OD effectively transfer to the equivalent AD problem with very minor loss of average return. Although this single empirical example does not constitute a universal guarantee, it highlights that the structural isomorphism can still provide a robust foundation for transfer even when environment dynamics violate transition independence. As expected, the performance of the learned policy and the zero-shot transfer degrade as the length of the delay increases. As a reference, an average reward of $0$ usually means that robots can stand still without falling and are barely able to move forward. An average reward of $150$ here means that agents consistently finish episodes by walking towards the goal. A completely failed episode would result in a $-110$ return.

\begin{figure}[htp]
    \centering
    \includegraphics[width=0.5\linewidth]{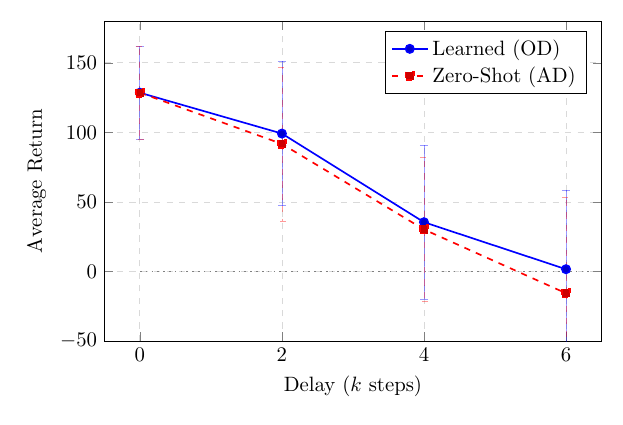}
    \caption{Performance of Zero-Shot policy transfer on the Multiwalker environment with 2-walkers with increasing delay. We use parameter sharing across agents for learning. For all delay values, agents have the same training budget of 20M steps.}
    \label{fig:zero-shot}
\end{figure}

\section{Discussion}

From the theoretical perspective, this work generalizes a well-known result for single-agent MDP to the partially observable and cooperative multi-agent setting in any horizon with a strict trajectory equivalence. Our contribution also lies in our formal set of conditions, highlighting some previously hidden assumptions, such as initialization artifacts. Therefore, our results for cooperative MARL are also interesting for Single-Agent systems and generally offer new insights on the equivalence of delay regimes in RL. We only address integer, constant delays. Thus, it remains restrictive but still has some practicality, as mentioned through our numerical experiments and from existing work on delay-aware systems \citep{chen2021delay,liu2024delay,monroc2022delay}. Introducing stochastic delays would firstly break Assumption 1 (FIFO action buffer), and is highly non-trivial in Dec-POMDPs settings as compared to existing results from single-agent RL literature.\\

As our approach for equivalence is purely formulated at the effective time, it sets aside the challenge of credit assignment through reward collection delay, either induced by pre-existing OD and AD or additional reward delay. This is very convenient from a theoretical view, and it separates the existence of equivalent optimal solutions in OD and AD regimes from the actual learning challenge specific to each regime. This separation supports a unified framework: since mixed delay configurations reduce to a pure OD system, we can focus on solving the OD problem. As shown in our zero-shot transfer experiments, a policy trained in this regime is valid for equivalent AD systems. Therefore, whether the same optimal policy can be learned under both delay regimes in any system (single or multi-agent) remains an open question.\\

\section{Conclusion}

In this work, we establish a fundamental theoretical framework for multi-agent reinforcement learning under delay. By formally decoupling the structural equivalence of observation and action delays from their practical learning dynamics, we resolve ambiguities in the literature. We demonstrate that while OD and AD yield isomorphic joint-policy spaces and identical trajectory distributions under appropriate initialization, the algorithmic challenges introduced by coupled non-stationarity and causal credit assignment render their learning processes fundamentally different in practice.\\

These results are intended to serve as a rigorous mathematical foundation for future algorithmic development and theoretical guarantees in delayed RL. Furthermore, while our current analysis is grounded in cooperative Dec-POMDPs, the state-action-observation isomorphism we establish is fundamentally agnostic to the reward structure. This strongly suggests that our equivalence framework extends naturally beyond cooperative tasks to competitive or general-sum Markov Games. Future work should explore this generalization and focus on developing algorithms that acknowledge differences between OD and AD learning and overcome challenges in highly coupled, mixed-delay systems.


\subsection*{Acknowledgments}
This work is carried out in the framework of AI-NRGY (Grant ref.: ANR-22-PETA-0004) and REDEEM (Grant ref.: ANR-23-PEIA-005) projects funded by France 2030. 

\clearpage
\vskip 0.2in
\bibliographystyle{unsrtnat}
\bibliography{main}

\newpage

\appendix
\section{Algorithms for Experimental Setup}\label{suppl:exp-setup}

We give the algorithms describing the dynamics of the environments (\ref{alg:env_step}), the delay buffer logic (\ref{alg:delay_step}), and the initialization mechanisms for AD regime (\ref{alg:warm_start}).

\begin{algorithm}[H]
\caption{Base Environment Step: Strongly Coupled Grid World}\label{alg:env_step}
\begin{algorithmic}[1]
\State \textbf{Input:} Joint actions $\mathbf{a} = (a^1, a^2)$, Current positions $\mathbf{p} = (p^1, p^2)$
\State \textbf{Parameters:} Grid size $N$, Bottleneck gap $g$, Radius $r$
\State \textbf{Output:} Next positions $\mathbf{p}_{new}$, Reward $R$, Terminal flag \text{Done}

\Statex \textbf{\% 1. Proposed Movement \& Wall Constraints}
\For{$i \in \{1, 2\}$}
    \State $p^i_{target} \leftarrow \text{Clip}(p^i + \text{move}(a^i), [0, N-1])$
    \If{$p^i_{target}$ is Wall \textbf{and} $p^i_{target} \neq g$}
        \State $p^i_{new} \leftarrow p^i$ \Comment{Revert if hitting a wall}
    \Else
        \State $p^i_{new} \leftarrow p^i_{target}$
    \EndIf
\EndFor

\Statex \textbf{\% 2. Agent Collision \& Repel Logic}
\If{$\| p^1_{new} - p^2_{new} \|_2 \leq r$}
    \State $\mathbf{p}_{new} \leftarrow \mathbf{p}$ \Comment{Cancel move for both agents}
    \If{$\| p^1_{new} - p^2_{new} \|_2 \leq r$} \Comment{If still overlapping}
        \State $v \leftarrow p^1_{new} - p^2_{new}$
        \State $dir \leftarrow \text{sign}(v)$ (or $(1, 1)$ if $v=0$)
        \For{$i \in \{1, 2\}$}
            \State $p_{push} \leftarrow \text{Clip}(p^i_{new} \pm dir, [0, N-1])$
            \If{$p_{push}$ is not Wall}
                \State $p^i_{new} \leftarrow p_{push}$ \Comment{Apply repel force}
            \EndIf
        \EndFor
    \EndIf
\EndIf

\Statex \textbf{\% 3. Reward \& Termination}
\State $R \leftarrow -1$
\If{$p^1_{new} == \text{Target}^1$ \textbf{and} $p^2_{new} == \text{Target}^2$}
    \State $R \leftarrow 10$, $\text{Done} \leftarrow \text{True}$
\ElsIf{StepCount $\geq$ MaxSteps}
    \State $\text{Done} \leftarrow \text{True}$
\EndIf
\State \Return $\mathbf{p}_{new}, R, \text{Done}$
\end{algorithmic}
\end{algorithm}

\clearpage

\begin{algorithm}[H]
\caption{Wrapper Step: Asymmetric Action \& Observation Delay}\label{alg:delay_step}
\begin{algorithmic}[1]
\State \textbf{Input:} Joint selected actions $\mathbf{a}_t = (a_t^1, a_t^2)$
\State \textbf{Parameters:} Delays $\mathbf{k} = (k_1, k_2)$, Mode $M \in \{\text{OD}, \text{AD}\}$
\State \textbf{State Variables:} Action Buffers $\mathcal{B}_{act}^i$, Obs Buffers $\mathcal{B}_{obs}^i$

\Statex \textbf{\% 1. Determine Effective Actions}
\For{$i \in \{1, 2\}$}
    \If{$k_i = 0$}
        \State $a_{exec}^i \leftarrow a_t^i$
    \ElsIf{$M = \text{AD}$}
        \State $a_{exec}^i \leftarrow \text{Dequeue}(\mathcal{B}_{act}^i)$ \Comment{Execute oldest action}
        \State $\text{Enqueue}(\mathcal{B}_{act}^i, a_t^i)$ \Comment{Queue new action}
    \ElsIf{$M = \text{OD}$}
        \State $a_{exec}^i \leftarrow a_t^i$ \Comment{Execute immediately}
        \State $\text{Enqueue}(\mathcal{B}_{act}^i, a_t^i)$ \Comment{Queue for history context}
    \EndIf
\EndFor

\Statex \textbf{\% 2. Environment Transition}
\State $s_{t+1}, R, \text{Done} \leftarrow \text{BaseEnv.Step}(a_{exec}^1, a_{exec}^2)$
\For{$i \in \{1, 2\}$}
    \State $\text{Enqueue}(\mathcal{B}_{obs}^i, s_{t+1}^i)$
\EndFor

\Statex \textbf{\% 3. Construct Augmented State $\tilde{s}_{t+1}$}
\For{$i \in \{1, 2\}$}
    \If{$k_i = 0$ \textbf{or} $M = \text{AD}$}
        \State $v^i \leftarrow s_{t+1}^i$ \Comment{Agent sees true current state}
    \ElsIf{$M = \text{OD}$}
        \State $v^i \leftarrow \text{Front}(\mathcal{B}_{obs}^i)$ \Comment{Agent sees delayed state $s_{t+1-k_i}^i$}
    \EndIf
    \State $\tilde{s}_{t+1}^i \leftarrow (v^i, \mathcal{B}_{act}^i)$
\EndFor

\State \Return $\tilde{s}_{t+1}, R, \text{Done}$
\end{algorithmic}
\end{algorithm}

\begin{algorithm}[H]
\caption{Wrapper Initialization (Warm vs. Cold Start)}\label{alg:warm_start}
\begin{algorithmic}[1]
\State \textbf{Input:} Initial Policy $\pi_0$ (Optional)
\State \textbf{Output:} Initial Augmented State $\tilde{s}_0$
\State $s_0 \leftarrow \text{BaseEnv.Reset}()$

\Statex \textbf{\% 1. Cold Initialization (Default actions)}
\For{$i \in \{1, 2\}$}
    \State Fill $\mathcal{B}_{obs}^i$ with $k_i + 1$ copies of $s_0^i$
    \State Fill $\mathcal{B}_{act}^i$ with $k_i$ copies of $a_{default}$ (Stay)
\EndFor

\Statex \textbf{\% 2. Warm Start Rollout}
\If{$\pi_0$ is provided}
    \State $k_{max} \leftarrow \max(k_1, k_2)$
    \For{$step = 1 \dots k_{max}$}
        \State $\tilde{s}_{current} \leftarrow \text{GetAugmentedState}(s_0)$
        \State $(a^1, a^2) \leftarrow \pi_0(\tilde{s}_{current})$
        \For{$i \in \{1, 2\}$}
            \If{$k_i > 0$}
                \State $\text{Enqueue}(\mathcal{B}_{act}^i, a^i)$
            \EndIf
        \EndFor
    \EndFor
\EndIf
\State \Return $\text{GetAugmentedState}(s_0)$
\end{algorithmic}
\end{algorithm}

\clearpage

\section{Experiments Parameters}

\subsection{Grid World Experiments}
\begin{table}[h!]
    \centering
    \renewcommand{\arraystretch}{1.2}
    \begin{tabular}{@{}ll@{}}
        \toprule
        \textbf{Hyperparameter} & \textbf{Value} \\
        \midrule
        \multicolumn{2}{@{}l}{\textit{Q-Learning Parameters}} \\
        Learning Rate ($\alpha$) & $0.1$ \\
        Discount Factor ($\gamma$) & $0.99$ \\
        Initial Exploration ($\epsilon_{\text{start}}$) & $1.0$ \\
        Minimum Exploration ($\epsilon_{\text{min}}$) & $0.05$ \\
        Exploration Decay Schedule & Linear over first $20\%$ of episodes \\
        \midrule
        \multicolumn{2}{@{}l}{\textit{Environment Parameters}} \\
        Grid Size & $6 \times 6$ \\
        Max Steps per Episode & $50$ \\
        Delay Configuration ($\mathbf{k}$) & $(0, 3)$ \\
        Collision Radius & $1$ \textit{(Coupled setting only)} \\
        \midrule
        \multicolumn{2}{@{}l}{\textit{Training \& Evaluation Parameters}} \\
        Total Episodes (Coupled Environment) & $600,000$ \\
        Total Episodes (Independent Environment) & $60,000$ \\
        Number of Independent Runs (Seeds) & $10$ \\
        Evaluation Frequency & Every $10$ episodes \\
        \bottomrule
    \end{tabular}
    \caption{Hyperparameters used for the tabular Q-Learning experiments.}
    \label{tab:hyperparameters}
\end{table}

\begin{table}[h!]
    \centering
    \renewcommand{\arraystretch}{1.2}
    \begin{tabular}{@{}ll@{}}
        \toprule
        \textbf{Hyperparameter} & \textbf{Value} \\
        \midrule
        \multicolumn{2}{@{}l}{\textit{Q-Learning Parameters}} \\
        Learning Rate ($\alpha$) & $0.1$ \\
        Discount Factor ($\gamma$) & $0.99$ \\
        Initial Exploration ($\epsilon_{\text{start}}$) & $1.0$ \\
        Minimum Exploration ($\epsilon_{\text{min}}$) & $0.05$ \\
        Exploration Decay Schedule & Linear over first $20\%$ of episodes \\
        \midrule
        \multicolumn{2}{@{}l}{\textit{Environment Parameters}} \\
        Grid Size & $7 \times 7$ \\
        Max Steps per Episode & $50$ \\
        Delay Configuration ($\mathbf{k}$) & $(0, 3)$ \\
        Coupling Dynamics & Transition Independent (No Collisions) \\
        \midrule
        \multicolumn{2}{@{}l}{\textit{Training \& Evaluation Parameters}} \\
        Total Episodes & $100,000$ \\
        Number of Independent Runs (Seeds) & $10$ \\
        Evaluation Frequency & Every $10$ episodes \\
        Initialization Modes Tested & OD (Warm), AD (Warm), AD (Cold) \\
        \bottomrule
    \end{tabular}
    \caption{Hyperparameters used for the Warm vs. Cold Start tabular Q-Learning experiments.}
    \label{tab:hyperparameters_exp2}
\end{table}

\begin{table}[h!]
    \centering
    \renewcommand{\arraystretch}{1.2}
    \begin{tabular}{@{}ll@{}}
        \toprule
        \textbf{Hyperparameter} & \textbf{Value} \\
        \midrule
        \multicolumn{2}{@{}l}{\textit{Q-Learning Parameters}} \\
        Learning Rate ($\alpha$) & $0.1$ \\
        Discount Factor ($\gamma$) & $0.99$ \\
        Initial Exploration ($\epsilon_{\text{start}}$) & $1.0$ \\
        Minimum Exploration ($\epsilon_{\text{min}}$) & $0.05$ \\
        Exploration Decay Schedule & Linear over first $20\%$ of episodes \\
        \midrule
        \multicolumn{2}{@{}l}{\textit{Environment Parameters}} \\
        Grid Size & $7 \times 7$ \\
        Max Steps per Episode & $50$ \\
        Delay Configuration ($\mathbf{k}$) & $(0, 3)$ \\
        Coupling Dynamics & Transition Independent (No Collisions) \\
        \midrule
        \multicolumn{2}{@{}l}{\textit{Training \& Evaluation Parameters}} \\
        Total Episodes & $100,000$ \\
        Number of Independent Runs (Seeds) & $10$ \\
        Evaluation Frequency & Every $10$ episodes \\
        Initialization Strategy & Warm Start (for all modes) \\
        Modes Tested & OD (Standard), AD (Standard), AD (Aligned) \\
        \bottomrule
    \end{tabular}
    \caption{Hyperparameters and algorithmic configuration used for the Effective-Time Alignment Q-Learning experiments.}
    \label{tab:hyperparameters_exp3}
\end{table}

\clearpage
\subsection{Multiwalker PPO Experiments}

\begin{table}[h!]
    \centering
    \renewcommand{\arraystretch}{1.2}
    \begin{tabular}{@{}ll@{}}
        \toprule
        \textbf{Hyperparameter} & \textbf{Value} \\
        \midrule
        \multicolumn{2}{@{}l}{\textit{PPO Algorithm Parameters}} \\
        Learning Rate & Linear decay, initial: $1 \times 10^{-4}$ \\
        Clip Range & Linear decay, initial: $0.2$ \\
        Horizon ($n_{\text{steps}}$) & $2048$ \\
        Batch Size & $128$ \\
        Number of Epochs & $10$ \\
        Target KL Divergence & $0.015$ \\
        Entropy Coefficient & $0.01$ \\
        GAE Lambda ($\lambda$) & $0.95$ \\
        Value Function Coefficient & $0.5$ \\
        Max Gradient Norm & $0.5$ \\
        \midrule
        \multicolumn{2}{@{}l}{\textit{Network Architecture}} \\
        Feature Extractor & Linear (to $128$) $\to$ LayerNorm $\to$ Tanh \\
        Actor Network (pi) & MLP: $[128, 128]$ with Tanh activation \\
        Critic Network (vf) & MLP: $[128, 128]$ with Tanh activation \\
        \midrule
        \multicolumn{2}{@{}l}{\textit{Environment Parameters (Multiwalker-v9)}} \\
        Number of Walkers & $2$ \\
        Max Cycles & $500$ \\
        Position Noise & $1 \times 10^{-3}$ \\
        Rewards (Forward / Fall / Terminate) & $1.0$ / $-10.0$ / $-100.0$ \\
        Shared Reward & True \\
        \midrule
        \multicolumn{2}{@{}l}{\textit{Training \& Zero-Shot Transfer Parameters}} \\
        Training Modality & Observation Delay (OD) \\
        Max Training Timesteps & $20,000,000$ \\
        Early Stopping Threshold & $200.0$ (Checked every $50,000$ steps) \\
        Zero-Shot Transfer Evaluation & $100$ episodes on Action Delay (AD) \\
        \bottomrule
    \end{tabular}
    \caption{Hyperparameters and environment configuration for the PPO Multiwalker experiments.}
    \label{tab:hyperparameters_multiwalker}
\end{table}

\end{document}